\documentclass[10pt,a4paper]{article}
\usepackage{amsmath}
\usepackage{epsfig,amsthm}
\usepackage{amssymb,amsfonts}
\usepackage{dsfont}
\usepackage{graphicx}
\usepackage{indentfirst}
\usepackage{mathrsfs}
\usepackage{color}
\usepackage{subfig}
\usepackage{lineno}
\usepackage{url}
\usepackage{longtable}
\usepackage{xparse}
\usepackage{natbib}
\usepackage{bm}
\usepackage{changes}

\newcommand{\D}{\mathrm{d}}
\newcommand{\I}{\mathrm{i}}
\newcommand{\E}{\mathrm{e}}

\newcommand{\cc}{\mathrm{c.c.}}
\newcommand{\R}{\mathbb{R}}

\newcommand*\diff{\mathop{}\!\D}

\newcommand{\abs}[1]{\lvert#1\rvert} 
\newcommand{\Abs}[1]{\left\lvert#1\right\rvert} 
\newcommand{\norm}[1]{\lVert#1\rVert}
\newcommand{\Norm}[1]{\left\lVert#1\right\rVert}

\newtheorem{thm}{Theorem}
\newtheorem{cor}{Corollary}
\newtheorem{defi}{Definition}
\newtheorem{lem}{Lemma}

\newtheorem{rmk}{Remark}
\newtheorem{assump}{Assumption}
\newtheorem{exam}{Example}

\newcommand*\patchAmsMathEnvironmentForLineno[1]{%
  \expandafter\let\csname old#1\expandafter\endcsname\csname #1\endcsname
  \expandafter\let\csname oldend#1\expandafter\endcsname\csname end#1\endcsname
  \renewenvironment{#1}%
     {\linenomath\csname old#1\endcsname}%
     {\csname oldend#1\endcsname\endlinenomath}}%
\newcommand*\patchBothAmsMathEnvironmentsForLineno[1]{%
  \patchAmsMathEnvironmentForLineno{#1}%
  \patchAmsMathEnvironmentForLineno{#1*}}%
\AtBeginDocument{%
\patchBothAmsMathEnvironmentsForLineno{equation}%
\patchBothAmsMathEnvironmentsForLineno{align}%
\patchBothAmsMathEnvironmentsForLineno{flalign}%
\patchBothAmsMathEnvironmentsForLineno{alignat}%
\patchBothAmsMathEnvironmentsForLineno{gather}%
\patchBothAmsMathEnvironmentsForLineno{multline}%
}

\title{Theory of the Frequency Principle for General Deep Neural Networks}

\author{
Tao Luo\\
luo196@purdue.edu\\
Department of Mathematics, Purdue University\\
\\
Zheng Ma\\
ma531@purdue.edu\\
Department of Mathematics, Purdue University\\
\\
Zhi-Qin John Xu\\
zhiqinxu@nyu.edu\\
New York University Abu Dhabi and Courant Institute \\
\\
Yaoyu Zhang \\
yaoyu@cims.nyu.edu\\
New York University Abu Dhabi  and Courant Institute\\
}

\begin{document}

\maketitle
\allowdisplaybreaks

\begin{abstract}
  Along with fruitful applications of Deep Neural Networks (DNNs) to realistic problems, recently, some empirical studies of DNNs reported a universal phenomenon of Frequency Principle (F-Principle): a DNN tends to learn a target function from low to high frequencies during the training. The F-Principle has been very useful in providing both qualitative and quantitative understandings of DNNs. In this paper, we rigorously investigate the F-Principle for the training dynamics of a general DNN at three stages: initial stage, intermediate stage, and final stage. For each stage, a theorem is provided in terms of proper quantities characterizing the F-Principle. Our results are general in the sense that they work for multilayer networks with general activation functions, population densities of data, and a large class of loss functions. Our work lays a theoretical foundation of the F-Principle for a better understanding of the training process of DNNs. 
\end{abstract}

%


\section{Introduction}

Deep learning has achieved great success as in many fields \citep{lecun2015deep},
e.g., speech recognition \citep{amodei2016deep}, object recognition
\citep{eitel2015multimodal}, natural language processing \citep{young2018recent}
and computer game control \citep{mnih2015human}. It has also been
adopted into algorithms to solve scientific computing problems \citep{weinan2017deep,khoo2017solving,he2018relu,fan2018multiscale}.
In principle, the universal approximation theorem states that a commonly-used
Deep Neural Network (DNN) of sufficiently large width can approximate
any function to a desired precision \citep{cybenko1989approximation}.
However, it remains a mystery that how a DNN finds a minimum corresponding to such an approximation through
the gradient-based training process.
To understand the learning behavior of DNNs for the approximation problem, recent works model the gradient flow
of parameters in a two-layer ReLU neural networks by a
partial differential equation (PDE) in the mean-field limit \citep{rotskoff2018parameters,mei2018mean,sirignano2018mean}.
However, it is not clear whether this PDE approach, which describes a neural network of one hidden layer of infinite width, can be extended
to general DNNs of multiple hidden layers and limited neuron number.

In this work, we take another approach that uses Fourier analysis to study the learning behavior of DNNs based on the phenomenon of \emph{Frequency Principle (F-Principle), i.e., a
DNN tends to learn a target function from low to high frequencies
during the training }\citep{xu_training_2018,rahaman2018spectral,xu2018understanding,xu2018frequency,xu2019frequency,zhang_explicitizing_2019}.
Empirically,
the F-Principle can be widely observed in
general DNNs for both benchmark and synthetic data \citep{xu_training_2018,xu2019frequency}.
Conceptually, it provides a qualitative explanation of the success
and failure of DNNs \citep{xu2019frequency}. Based on the F-Principle, a series of works has been done. For example, it is  used as an important phenomenon to pursue
fundamentally different learning trajectories of meta-learning \citep{rabinowitz2019meta}.  It is also used as a tool to observe the performance of adaptive activation function \citep{jagtap2019adaptive}.
Based on the F-Principle, a numerical algorithm is developed to accelerate the DNN fitting of high frequency functions by shifting high frequencies
to lower ones \citep{cai2019phasednn}.
Theoretically, an effective model of linear F-Principle
dynamics \citep{zhang_explicitizing_2019}, which accurately predicts the
learning results of two-layer ReLU neural networks of large widths, leads to an apriori estimate of the generalization bound. In addition, a theorem is provided for the characterization of the initial
training stage of a two-layer $\tanh$ network \citep{xu2019frequency}. The same theoretical analysis in \cite{xu2019frequency}
is also adopted in the analysis of DNNs with ReLU activation function
\citep{rahaman2018spectral} and a nonlinear collaborative scheme of
loss functions for DNN training \citep{zhen2018nonlinear}. These subsequent
works show the importance of the F-Principle. However,
a theory of the F-Principle for general
DNNs is still missing. 

Following the same direction as in \cite{xu2019frequency}, in this work, we
propose a theoretical framework of Fourier analysis for the study of 
the training behavior of \emph{general} DNNs in the following three stages:
the initial stage, the intermediate stage, and the final stage. 
At all stages, we rigorously characterize the F-Principle by estimating some proper quantities.
At the initial
and final stages with the MSE loss (mean-squared error, also known as $L^{2}$ loss), we show that
the change of MSE is dominated by low frequencies.
Furthermore, in these two stages with general $L^{p}$ ($2\leq p<\infty$) loss, we show that the change of the DNN
output is dominated by the low-frequency part. A key contribution
of this work is on the intermediate stage --- with $L^{p}$ loss, the difference of the MSE over a certain period, in which the MSE is reduced by half, is dominated by
the low frequencies. In summary, we verify that the F-Principle
is universal in the sense that our results not only work for DNNs
of multiple layers with any commonly-used activation function, e.g., ReLU, sigmoid,
and tanh, but also work for a general population density of data and for a
general class of loss functions. The key insight unraveled by our analysis is that the regularity of DNN converts into the decay
rate of a loss function in the frequency domain. 

\section{Preliminaries}
We start with a brief introduction to DNNs and its training dynamics. Under very mild assumptions, we provide some regularity results which are crucial to the proof of the main theorems summarized in the next section. 

\subsection{Deep Neural Networks}
Consider a DNN with $(H-1)$-hidden layers and general activation functions. We regard the $d$-dimensional input as the $0$-th layer and the one-dimensional output as the $H$-th layer. Let $n_l$ be the number of neurons in the $l$-th layer. In particular, $n_0=d$ and $n_H=1$.

The {\em hypothesis space} $\mathcal{H}$ is a family of \emph{hypothesis functions} parametrized by the \emph{parameter vector} $\theta\in\R^N$ whose entries are called \emph{parameters} $W^{(l)}_i$'s (also known as \emph{weight}) and $b^{(l)}_i$'s (also known as \emph{bias}). More precisely, we set
\begin{align}
  \theta=\Big(W^{(l)},b^{(l)}\Big)_{l=1}^{H},
\end{align}
where for $l=1,\cdots,H$,
\begin{align}
  W^{(l)}
  &= \Big(W^{(l)}_i\Big)_{i=1}^{n_l},\quad W^{(l)}_i\in\R^{n_{l-1}}\\
  b^{(l)}
  &= \Big(b^{(l)}_i\Big)_{i=1}^{n_l},\quad b^{(l)}\in\R.
\end{align}
The \emph{size} $N$ of the network is the number of the parameters, i.e.,
\begin{equation}
  N=\sum_{l=0}^{H-1}(n_l+1) n_{l+1}.\label{eq..size}
\end{equation}

To define the hypothesis functions in $\mathcal{H}$, we need some nonlinear functions which are known as \emph{activation functions}:
\begin{equation}
  \sigma_i^{(l)}:\R\to\R,\quad l=1,\cdots,H-1,\quad i=1,\cdots,n_l.
\end{equation}
Given $\theta\in \R^N$, the corresponding function $h$ in $\mathcal{H}$ is defined by a series of function compositions. First, we set $h^{(0)}=\mathrm{id}:\R^{d}\to\R^{d}$, i.e., $h^{(0)}(x)=x$ for all $x\in\R^d$. Then for $l=1,\cdots,H-1$, $h^{(l)}$ is defined recursively as
\begin{align}
  &h^{(l)}:\R^{d}\to\R^{n_l},\\
  &(h^{(l)}(x))_i=\sigma_{i}^{(l)} (W_i^{(l)}\cdot h^{(l-1)}(x)+b_i^{(l)}),\quad i=1,\cdots,n_l.
\end{align}
Finally, we denote 
\begin{equation}
  h^{(H)}(x)=W^{(H)}\cdot h^{(H-1)}(x)+b^{(H)}.
\end{equation}
We remark that for the most applications, the activation functions $\sigma_i^{(l)}$ are chosen to be the same, i.e., $\sigma_i^{(l)}=\sigma$, $l=1,\cdots,H-1$, $i=1,\cdots,n_l$.

\begin{exam}
  For instance, if a one-hidden layer neural network is used, then $H=2$ and the hypothesis function can be written into the following form:
  \begin{equation}
  h^{(2)}(x,\theta)=\sum_{i=1}^{n}w_{i}^{(2)}\sigma(w_{i}^{(1)}\cdot x+b_{i}^{(1)}),\quad w_{i}^{(2)},b_{i}^{(1)}\in{\rm \R}, w_{i}^{(1)}\in\R^d.
  \end{equation}
  Thus the size of the network $N=(d+2)n$ which is consistent with \eqref{eq..size}.
\end{exam}

We are only interested in the target function $f_\mathrm{target}$ in a compact domain $\Omega$, i.e., $\Omega\subset\subset\R^d$. A bump function $\chi$ is used to truncate both hypothesis and target functions:
\begin{align}
    h(x,\theta)&= h^{(H)}(x,\theta)\chi(x),\\
    f(x)&= f_\mathrm{target}(x)\chi(x).
\end{align}
In the sequel, we will also refer to $h$ and $f$ as the hypothesis and target functions, respectively. 

\subsection{Loss Function and Training Dynamics}
In this work, we investigate the training dynamics of parameters in DNNs with two cases of loss functions:

(i) The MSE loss function with population measure $\mu$, i.e., 
\begin{equation}
  L_\rho(\theta)=\int_{\R^d}\abs{h(x,\theta)-f(x)}^2\diff{\mu}.
\end{equation}
In this case, the training dynamics of $\theta$ follows the gradient flow:
\begin{equation}
  \left\{
  \begin{array}{l}
  \dfrac{\D \theta}{\D t}=-\nabla_\theta L_\rho\\
  \theta(0)=\theta_0.
  \end{array}
  \right.\label{eq..TrainingL2Loss}
\end{equation}

(ii) A general loss function with population measure $\mu$, i.e., 
\begin{equation}
    \tilde{L}_\rho(\theta)=\int_{\R^d}\ell(h(x,\theta)-f(x))\diff{\mu},
\end{equation}
where the function $\ell$ satisfies some mild assumptions to be explained later.
In this case, the training dynamics of $\theta$ becomes:
\begin{equation}
  \left\{
  \begin{array}{l}
  \dfrac{\D \theta}{\D t}=-\nabla_\theta \tilde{L}_\rho\\
  \theta(0)=\theta_0.
  \end{array}
  \right.\label{eq..TrainingGeneralLoss}
\end{equation}

In the case of MSE loss function, we have
\begin{align}
  L_\rho
  &= \int_{\R^d}\abs{h_\rho(x)-f_\rho(x)}^2\diff{x}\\
  &= \int_{\R^d}\abs{\hat{h}_\rho(\xi)-\hat{f}_\rho(\xi)}^2\diff{\xi},
\end{align}
where $\rho$, satisfying $\D \mu=\rho\D x$, is called the {\em population density} and
\begin{equation}
  h_\rho=h\sqrt{\rho},\quad f_\rho=f\sqrt{\rho}.
\end{equation}
The second equality is due to the Plancherel theorem. Here and in the sequel, we use the following conventions for the Fourier transform and its inverse transform on $\R^d$:
\begin{equation*}
  \mathcal{F}[g](\xi)=\hat{g}(\xi)=\int_{\R^d}g(x)\E^{-2\pi\I \xi\cdot x}\diff{x},\quad
  g(x)=\int_{\R^d}\hat{g}(\xi)\E^{2\pi\I \xi\cdot x}\diff{\xi}.
\end{equation*}
For the convenience of proofs, we denote
\begin{align}
  L_\rho(\theta)
  &= \int_{\R^d}q_\rho(\xi,\theta)\diff{\xi},\\
  q_\rho(\xi,\theta)
  &= \abs{\hat{h}_\rho(\xi,\theta)-\hat{f}_\rho(\xi)}^2.
\end{align}

\subsection{Assumptions}

The requirements on $\chi$, $f$, $\sigma$, and $\mu$ are summarized here.
\begin{assump}[regularity]\label{assump..Differentiability}
  The bump function $\chi$ satisfies $\chi(x)=1$, $x\in\Omega$ and $\chi(x)=0$, $x\in\R^d\backslash\Omega'$ for domains $\Omega$ and $\Omega'$ with $\Omega\subset\subset\Omega'\subset\subset\R^d$. There is a positive integer $k$ (can be $\infty$) such that $f_\mathrm{target}\in W^{k,\infty}_\mathrm{loc}(\R^d;\R)$,  $\chi\in W^{k,\infty}_\mathrm{loc}(\R^d;[0,+\infty))$
  , and $\sigma_i^{(l)}\in W^{k,\infty}_\mathrm{loc}(\R;\R)$ for $l=1,\cdots,H-1$, $i=1,\cdots,n_l$. 
\end{assump}
\begin{assump}[bounded population density]\label{assump..DensityBounded}
  There exists a function $\rho\in L^{\infty}(\R^d;[0,+\infty))$ satisfying $\D \mu=\rho\diff{x}$.
\end{assump}

\begin{exam}
  Here we list some commonly-used activation functions:\\
  (1) ReLU (Rectified Linear Unit): $\mathrm{ReLU}(x)=\max(0,x)$, $x\in\R$;\\
  (2) tanh (hyperbolic tangent): $\tanh(x)=\frac{\E^x-\E^{-x}}{\E^x+\E^{-x}}$, $x\in\R$;\\
  (3) sigmoid function (also known as logistic function): $S(x)=\frac{1}{1+\E^{-x}}$, $x\in\R$.
\end{exam}

\begin{rmk}
  It is also allowed that $k=\infty$ where the functions $f$ and $\sigma_i^{(l)}$ are all $C^\infty$ by Sobolev embedding inequalities. This case includes $\tanh$ and $\mathrm{sigmoid}$ activation functions.
\end{rmk}
\begin{rmk}
  If an activation function is ReLU, then $k=1$.
\end{rmk}
\begin{rmk}
  For $x\in\Omega$, we have $h(x,\theta)-f(x)=h^{(H)}(x,\theta)-f_\mathrm{target}(x)$.
\end{rmk}

For the training dynamics \eqref{eq..TrainingL2Loss} or \eqref{eq..TrainingGeneralLoss}, we suppose the parameters are bounded.
\begin{assump}[bounded trajectory]\label{assump..BoundedTrajectory}
  The training dynamics is nontrivial, i.e., $\theta(t)\not\equiv\mathrm{const}$. There exists a constant $R>0$ such that $\sup_{t\geq 0}\abs{\theta(t)}\leq R$ where the parameter vector $\theta(t)$ is the solution to \eqref{eq..TrainingL2Loss} or \eqref{eq..TrainingGeneralLoss}. 
\end{assump}
\begin{rmk}
  The bound $R$ depends on initial parameter $\theta_0$.
\end{rmk}
In the case of MSE loss function, we will further take the following assumption.
\begin{assump}\label{assump..LeastSquareLoss}
    The density $\rho$ satisfies $\sqrt{\rho}\in W^{k,\infty}_\mathrm{loc}(\R^d;[0,+\infty))$. 
\end{assump}
The general loss function considered in this work satisfies the following assumption.
\begin{assump}[general loss function]\label{assump..GeneralLoss}
  The function $\ell$ in the general loss function $\tilde{L}_\rho(\theta)$ satisfies $\ell\in C^{2}(\R;[0,+\infty))$ and there exist positive constants $C$ and $r_0$ such that $C^{-1}[\ell'(z)]^2\leq \ell(z)\leq C \abs{z}^2$ for $\abs{z}\leq r_0$.
\end{assump}
\begin{exam}
  The $L^p$ ($2\leq p<\infty$) loss function satisfies Assumption \ref{assump..GeneralLoss}. Here the $L^p$ ($1\leq p<\infty$) loss functions used in machine learning are defined as $L_\rho(\theta)=\int_{\R^d}\abs{h(x,\theta)-f(x)}^p\rho(x)\diff{x}$ which is a little bit different from the $L^p$ norm used in mathematics. 
\end{exam}

\subsection{Regularity}

We begin with the integrability of the hypothesis function. To achieve this, we use the ``Japanese bracket'' of $\xi$:
\begin{equation}
  \langle\xi\rangle:=(1+\abs{\xi}^2)^{1/2}.
\end{equation}
\begin{lem}\label{lem..LinfinityIntegrability}
  Suppose that the Assumption \ref{assump..Differentiability} holds. Given any $\theta\in\R^N$, the hypothesis function $h\in W^{k,2}(\R^d;\R)$ and its gradient with respect to the parameters $\nabla_\theta h\in W^{k-1,2}(\R^d;\R^N)$. Also, we have $f\in W^{k,2}(\R^d;\R)$.
\end{lem}
\begin{proof}
  Recall that $f(x)=f_\mathrm{target}(x)\chi(x)$ and $h(x)=h^{(H)}(x)\chi(x)$ given $\theta\in\R^N$. By Assumption \ref{assump..Differentiability}, $f_\mathrm{target}\in W^{k,\infty}_\mathrm{loc}(\R^d;\R)$ and $\chi(x)\in W^{k,\infty}_\mathrm{loc}(\R^d;[0,+\infty))$ with a compact support. Thus $f\in W^{k,2}(\R^d;\R)$. In order to show $h\in W^{k,2}(\R^d;\R)$, it is sufficient to prove that
  $h^{(H)}\in W^{k,\infty}_\mathrm{loc}(\R^d;\R)$. Indeed, we prove $h^{(l)}\in W^{k,\infty}_\mathrm{loc}(\R^d;\R^{n_l})$ for $l=0,1,\cdots,H$ by induction. For $l=0$, $h^{(0)}\in W^{k,\infty}_\mathrm{loc}(\R^d;\R^{n_0})$ because $h^{(0)}(x)=x$ and $n_0=d$. Suppose that for $l$ ($0\leq l\leq H-2$) we have $h^{(l)}\in W^{k,\infty}_\mathrm{loc}(\R^d;\R^{n_l})$. Now let us consider $h^{(l+1)}$ with $(h^{(l+1)}(x))_i=\sigma_{i}^{(l+1)} (W_i^{(l+1)}\cdot h^{(l)}(x)+b_i^{(l+1)})$, $i=1,\cdots,n_{l+1}$. By the induction assumption, we have $W_i^{(l+1)}\cdot h^{(l)}+b_i^{(l+1)}\in W^{k,\infty}_\mathrm{loc}(\R^d;\R)$. 
  By Assumption \ref{assump..Differentiability}, $\sigma_i^{(l)}\in W^{k,\infty}_\mathrm{loc}(\R;\R)$. Note that $\sigma^{(l)}_{i}\in C^{k-1}(\R;\R)$ by Sobolev embedding. Then $(h^{(l+1)})_i\in W^{k,\infty}_\mathrm{loc}(\R^d;\R)$ because of the chain rule and the fact that the composition of continuous functions is still continuous. Finally, for $l=H-1$, we have $h^{(H)}=W^{(H)}\cdot h^{(H-1)}+b^{(H)}\in W^{k,\infty}_\mathrm{loc}(\R^d;\R)$.

  The proof for $\nabla_\theta h$ is similar if we note that $(\sigma_i^{(l)})' \in W^{k-1,\infty}_\mathrm{loc}(\R;\R)$.
\end{proof}

\begin{rmk}
  The continuity of $\sigma_i^{(l)}$ is neccesary because the composition of two Lebesgue measurable functions need not be Lebesgue measurable.
\end{rmk}

\begin{lem}\label{lem..L2Integrability}
    Suppose that the Assumptions \ref{assump..Differentiability} and \ref{assump..DensityBounded} hold. Then
    
    \noindent
  (a). For any $0\leq m\leq k$, we have
  \begin{align}
    \langle\cdot\rangle^m\abs{\hat{h}(\cdot,\theta)}
    \in L^2(\R^d;\R),
    \label{eq..hhatL2}\\
    \langle\cdot\rangle^m\abs{\hat{f}(\cdot)}
    \in L^2(\R^d;\R).\label{eq..fhatL2}
  \end{align}
  (b). For any $0\leq m\leq k-1$, we have
  \begin{equation}
    \langle\cdot\rangle^{m}\abs{\nabla_\theta \hat{h}(\cdot,\theta)}
    \in L^2(\R^d;\R).\label{eq..gradienthhatL2}
  \end{equation}
  (c). For any $0\leq m\leq 2k-1$, we have
  \begin{equation}
    \langle\cdot\rangle^{m}\abs{\nabla_\theta q(\cdot,\theta)}
    \in L^1(\R^d;\R).\label{eq..gradientqL1}
  \end{equation}
\end{lem}
\begin{proof}
  (a). Let $0\leq m\leq k$. Given $\theta\in\R^N$, we have $f, h \in W^{k,2}(\R^d;\R)$ by Lemma \ref{lem..LinfinityIntegrability}. It is well known that for any function $g\in W^{k,2}(\R^d)$, for $0\leq m\leq k$,
  \begin{equation}
    C\norm{g}_{W^{m,2}(\R^d)}\leq \norm{\langle\cdot\rangle^m \abs{\hat{g}}}_{L^2(\R^d)}
    \leq \tilde{C}\norm{g}_{W^{m,2}(\R^d)},
  \end{equation}
  where the positive constants $C$ and $\tilde{C}$ only depend on $d$ and $m$.
  The statements \eqref{eq..hhatL2} and \eqref{eq..fhatL2} follow this.

  (b). Let $0\leq m\leq k-1$. Given $\theta\in\R^N$, we have $\nabla_\theta h\in W^{k-1,2}(\R^d;\R^N)$ by Lemma \ref{lem..LinfinityIntegrability}. Similar to part (a), this leads to \eqref{eq..gradienthhatL2}.

  (c). Let $m_1=m-m_2$ and $m_2=\min\{m, k\}$. Then $0\leq m_1\leq k-1$ and $0\leq m_2\leq k$.
  Combining the inequalities in parts (a) and (b), we have
  \begin{align}
    \norm{\langle\cdot\rangle^{m} \abs{\nabla_\theta q(\cdot,\theta)}}_{L^1(\R^d)}
    &= \Norm{\langle\cdot\rangle^{m}\Abs{\left(\nabla_\theta \hat{h}(\cdot,\theta)\right) \overline{\hat{h}(\cdot,\theta)-\hat{f}(\cdot)}+\cc}}_{L^1(\R^d)}\nonumber\\
    &\leq 2\Norm{\langle\cdot\rangle^{m_1}\abs{\nabla_\theta\hat{h}(\cdot,\theta)}}_{L^2(\R^d)}\Norm{\langle\cdot\rangle^{m_2}\abs{\hat{h}(\cdot,\theta)-\hat{f}(\cdot)}}_{L^2(\R^d)}\nonumber\\
    &< \infty.
  \end{align}
\end{proof}

\begin{lem}\label{lem..L2Integrabilityrho}
    Suppose that the Assumptions \ref{assump..Differentiability},  \ref{assump..DensityBounded}, and \ref{assump..LeastSquareLoss} hold. Then

    \noindent
  (a). For any $0\leq m\leq k$, we have
  \begin{align}
    \langle\cdot\rangle^m\abs{\hat{h}_\rho(\cdot,\theta)}
    \in L^2(\R^d;\R),
    \label{eq..hhatL2rho}\\
    \langle\cdot\rangle^m\abs{\hat{f}_\rho(\cdot)}
    \in L^2(\R^d;\R).\label{eq..fhatL2rho}
  \end{align}
  (b). For any $0\leq m\leq k-1$, we have
  \begin{equation}
    \langle\cdot\rangle^{m}\abs{\nabla_\theta \hat{h}_\rho(\cdot,\theta)}
    \in L^2(\R^d;\R).\label{eq..gradienthhatL2rho}
  \end{equation}
  (c). For any $0\leq m\leq 2k-1$, we have
  \begin{equation}
    \langle\cdot\rangle^{m}\abs{\nabla_\theta q_\rho(\cdot,\theta)}
    \in L^1(\R^d;\R).\label{eq..gradientqL1rho}
  \end{equation}
\end{lem}
\begin{proof}
    The proof is similar to the one of Lemma \ref{lem..L2Integrability}. The only new ingredient is assumption that $\sqrt{\rho}\in W^{k,\infty}_\mathrm{loc}(\R^d;\R)$.
\end{proof}

\section{Main Results}
In this section, we first propose several quantitative characterization for the F-Principle. 
Main results are then summarized with numerical illustrations at the end of this section.

\subsection{Characterization of F-Principle}
For the MSE loss function, a natural quantity to characterize the F-principle is the ratio of the loss function decrements caused by low frequencies and the total loss function decrements. To achieve this, we devide the MSE loss function into two parts, contributed by low and high frequencies, respectively, i.e., 
\begin{equation}
    L^-_{\rho,\eta}(\theta)=\int_{B_\eta}q_\rho(\xi,\theta)\diff{\xi},\quad L^+_{\rho,\eta}(\theta)=\int_{B_\eta^c}q_\rho(\xi,\theta)\diff{\xi},
\end{equation}
where $B_\eta$ and $B_\eta^c=\R^d\backslash B_\eta$ are a ball centered at the origin with radius $\eta>0$ and its complement. Thus $L_\rho=L^-_{\rho,\eta}+L^+_{\rho,\eta}$ for any $\eta>0$.
The ratio considered for characterizing the F-Principle is
\begin{equation}
    \frac{\abs{\D L_{\rho,\eta}^-/\D t}}{\abs{\D L_\rho/\D t}}\quad \mathrm{and}\quad \frac{\abs{\D L_{\rho,\eta}^+/\D t}}{\abs{\D L_\rho/\D t}}.\label{eq..CharacterizationRatioRho}
\end{equation}

For a general loss function, the training dynamics leads to
\begin{equation}
    \frac{\D \tilde{L}_\rho}{\D t}=-\abs{\nabla_\theta \tilde{L}_\rho}^2.
\end{equation}
In this case, we study
\begin{equation}
    L(\theta)=\int_{\R^d}\abs{\hat{h}(\xi,\theta)-\hat{f}(\xi)}^2\diff{\xi}.
\end{equation}
We remark that for a given $\theta$, $L(\theta)=\int_{\R^d}\abs{h(x,\theta)-f(x)}^2\diff{x}$ has nothing to do with $\mu$.
We still take the decomposition $L=L^-_\eta+L^+_\eta$ with
\begin{equation}
    L^-_{\eta}(\theta)=\int_{B_\eta}q(\xi,\theta)\diff{\xi},\quad L^+_{\eta}(\theta)=\int_{B_\eta^c}q(\xi,\theta)\diff{\xi},
\end{equation}
where
\begin{equation}
    q(\xi,\theta)
    =\abs{\hat{h}(\xi,\theta)-\hat{f}(\xi)}^2.
\end{equation}
One can simply mimic \eqref{eq..CharacterizationRatioRho} and consider
\begin{equation}
    \frac{\abs{\D L_{\eta}^-/\D t}}{\abs{\D L/\D t}}\quad \mathrm{and}\quad \frac{\abs{\D L_{\eta}^+/\D t}}{\abs{\D L/\D t}}.\label{eq..CharacterizationRatio}
\end{equation}
However, there is an issue in this characterization: $L$ may not be monotonically decreasing and the denominator in \eqref{eq..CharacterizationRatio} may be zero. To overcome this, a time averaging is required. Indeed, we investigate the following ratio where integrals are taken for both numerator and denominator in \eqref{eq..CharacterizationRatio}:
\begin{equation}
    \frac{\int_{T_1}^{T_2}\Abs{\frac{\D L_{\eta}^-}{\D t}}\diff{t}}{\int_{T_1}^{T_2}\Abs{\frac{\D L}{\D t}}\diff{t}}
    \quad \mathrm{and}\quad \frac{\int_{T_1}^{T_2}\Abs{\frac{\D L_{\eta}^+}{\D t}}\diff{t}}{\int_{T_1}^{T_2}\Abs{\frac{\D L}{\D t}}\diff{t}}.\label{eq..CharacterizationRatioIntegral}
\end{equation}
For the general loss function, we also propose another quantity to characterize the F-Principle:
  \begin{equation}
    \frac{\norm{\D \hat{h}/\D t}_{L^2(B_\eta)}}
    {\norm{\D \hat{h}/\D t}_{L^2(\R^d)}}
    \quad\mathrm{and}\quad
    \frac{\norm{\D \hat{h}/\D t}_{L^2(B_\eta^c)}}
    {\norm{\D \hat{h}/\D t}_{L^2(\R^d)}}.
  \end{equation}

\subsection{Main Theorems}
As we mentioned in the introduction, the training dynamics of a DNN has three stages: initial stage, intermediate stage, and final stage. 
For each stage, we provide a theorem to characterize the F-Principle. 

\noindent
\underline{\bf Initial Stage}
\vspace{3mm}

We start with the F-Principle in the initial stage. Clearly, the constants $C$ in the estimates depend on the initial parameter $\theta_0$ and the time $T$.
\begin{thm}\label{thm..InitialStage}[F-Principle in the initial stage]\\
    \noindent
    \emph{($L^2$ loss function)}
    Suppose that Assumptions \ref{assump..Differentiability}, \ref{assump..DensityBounded}, \ref{assump..BoundedTrajectory}, and \ref{assump..LeastSquareLoss} hold. We consider the training dynamics \eqref{eq..TrainingL2Loss}. Then for any $1\leq m\leq 2k-1$ and any $T>0$ satifying $\abs{\nabla_\theta L_\rho(\theta(T))}>0$ (if $k=1$, we further require that $\inf_{t\in(0,T]}\abs{\nabla_\theta L_\rho(\theta(t))}>0$), there is a constant $C>0$ such that
    \begin{equation}
        \frac{\abs{\D L^+_{\rho,\eta}/\D t}}{\abs{\D L_\rho/\D t}}\leq C\eta^{-m}\quad\mathrm{and}\quad\frac{\abs{\D L^-_{\rho,\eta}/\D t}}{\abs{\D L_\rho/\D t}}\geq 1-C\eta^{-m},\quad t\in(0,T].
     \end{equation}
    \noindent
    \emph{(general loss function)}
    Suppose that Assumptions \ref{assump..Differentiability}, \ref{assump..DensityBounded}, \ref{assump..BoundedTrajectory}, and \ref{assump..GeneralLoss} hold. We consider the training dynamics \eqref{eq..TrainingGeneralLoss}. Then for any $1\leq m\leq k-1$ and any $T>0$ satifying $\abs{\nabla_\theta \tilde{L}_\rho(\theta(T))}>0$, there is a constant $C>0$ such that 
    \begin{equation}
      \frac{\norm{\D \hat{h}/\D t}_{L^2(B_\eta^c)}}
      {\norm{\D \hat{h}/\D t}_{L^2(\R^d)}}
      \leq C\eta^{-m}\quad\mathrm{and}\quad
      \frac{\norm{\D \hat{h}/\D t}_{L^2(B_\eta)}}
      {\norm{\D \hat{h}/\D t}_{L^2(\R^d)}}
      \geq 1-C\eta^{-m},\quad t\in(0,T].
    \end{equation}
\end{thm}

\noindent
\underline{\bf Intermediate Stage}
\vspace{3mm}

 The theorem of intermediate stage is superior to the other results (initial/final stage) in three aspects. First, for a general loss function considered here, Plancherel theorem is not helpful. It is even more challenging to show the F-Principle based on the $L^2$-characterization $L^{-}_\eta(\theta)=\int_{B_\eta}\abs{\hat{h}(\xi,\theta)-\hat{f}(\xi)}^2\diff{\xi}$ in the training dynamics which is a gradient flow of a non-$L^2$ loss function:
\begin{equation}
    \frac{\D \tilde{L}_\rho}{\D t}=-\abs{\nabla_\theta \tilde{L}_\rho}^2.
\end{equation}
Secondly, although $\tilde{L}_\rho(\theta(t))$ decays as $t$ increases, $L(\theta(t))$ may not be monotonically decreasing. As a result, $\frac{\D L}{\D t}$ might vanish and should not be used in the denominator of the ratio $\frac{\D L_\eta/\D t}{\D L/\D t}$. However, the ratio still makes sense if we replace the infinitesimal change by a finite decrements in both numerator and denominator (see the precise meaning in Eq. \eqref{eq..SqrtTGeneralLoss}). The particular choice of a finite decrement is indeed related to the time-scale of the training dynamics. A proper time-scale is the half-life $T_2-T_1$ satisfying $\frac{1}{2}L(\theta(T_1))=L(\theta(T_2))$. Thirdly, we obtain an upper bound for the dependence of training period $T_2-T_1$. This bound works for all the situations. If the non-degenerate global minimizer is obtained, the dependence on $T_2-T_1$ in Eq. \eqref{eq..SqrtTGeneralLoss} can also be removed and leads to a consistent result to the results for the final stage.

\begin{thm}\label{thm..IntermediateStage}[F-Principle in the intermediate stage]\\
    \noindent
    \emph{(general loss function)}
    Suppose that Assumptions \ref{assump..Differentiability}, \ref{assump..DensityBounded}, \ref{assump..BoundedTrajectory}, and \ref{assump..GeneralLoss} hold.
    We consider the training dynamics \eqref{eq..TrainingGeneralLoss}. 
    Then for any $1\leq m\leq k-1$, there is a constant $C>0$ such that for any $0<T_1<T_2$ satisfying $\frac{1}{2}L(\theta(T_1))\geq L(\theta(T_2))$, 
    we have
    \begin{equation}
      \frac{\int_{T_1}^{T_2}\Abs{\frac{\D L_{\eta}^+}{\D t}}\diff{t}}{\int_{T_1}^{T_2}\Abs{\frac{\D L}{\D t}}\diff{t}}\leq C\sqrt{T_2-T_1}\eta^{-m}.\label{eq..SqrtTGeneralLoss}
    \end{equation}
\end{thm}

\noindent
\underline{\bf Final Stage}
\vspace{3mm}

If non-degenerate global minimizers are achieved in the training dynamics, we can obtain global-in-time result which characterizing the training dynamics in the final stage. Here we give the definition for non-degenerate minimizers:
\begin{defi}
    A minimizer $\theta^*$ of $L_\rho$ (or $\tilde{L}_\rho$, respectively) is global if $L_\rho(\theta^*)=0$ (or $\tilde{L}_\rho(\theta^*)=0$, respectively). The minimizer is non-degenerate if the Hessian matrix $\nabla_\theta^2 L_\rho(\theta^*)$ (or $\nabla_\theta^2\tilde{L}_\rho(\theta^*)$, respectively) exists and is positive definite.
\end{defi}

\begin{thm}\label{thm..FinalStage}[F-Principle in the final stage]\\
    \noindent
    \emph{($L^2$ loss function)}
    Suppose that Assumptions \ref{assump..Differentiability},  \ref{assump..DensityBounded}, \ref{assump..BoundedTrajectory}, and \ref{assump..LeastSquareLoss} hold. We consider the training dynamics \eqref{eq..TrainingL2Loss}.  If the solution $\theta$ converges to a non-degenerate global minimizer $\theta^*$, then for any $1\leq m\leq k-1$, there is a constant $C>0$ such that
    \begin{equation}
      \frac{\abs{\D L^+_{\rho,\eta}/\D t}}{\abs{\D L_\rho/\D t}}\leq C\eta^{-m}\quad\mathrm{and}\quad\frac{\abs{\D L^-_{\rho,\eta}/\D t}}{\abs{\D L_\rho/\D t}}\geq 1-C\eta^{-m},\quad t\in(0,+\infty).
    \end{equation}
    \noindent
    \emph{(general loss function)}
    Suppose that Assumptions \ref{assump..Differentiability}, \ref{assump..DensityBounded}, \ref{assump..BoundedTrajectory}, and \ref{assump..GeneralLoss} hold. We consider the training dynamics \eqref{eq..TrainingGeneralLoss}. If the solution $\theta$ converges to a non-degenerate global minimizer $\theta^*$, then for any $1\leq m\leq k-1$, there is a constant $C>0$ such that 
    \begin{equation}
      \frac{\norm{\D \hat{h}/\D t}_{L^2(B_\eta^c)}}
      {\norm{\D \hat{h}/\D t}_{L^2(\R^d)}}
      \leq C\eta^{-m}\quad\mathrm{and}\quad
      \frac{\norm{\D \hat{h}/\D t}_{L^2(B_\eta)}}
      {\norm{\D \hat{h}/\D t}_{L^2(\R^d)}}
      \geq 1-C\eta^{-m},\quad t\in(0,+\infty).
    \end{equation}
\end{thm}

\subsection{Discussion and Illustrations}

To help the readers get some intuitions of the above theorems, we present a numerical example using the following target function
\begin{equation*}
    f(x)=\sum_{j=1}^{500}\sin(jx/10)/j.
\end{equation*}

The training data are uniformly sampled from $[-3.14,3.14]$ with
sample size $300$. The discrete Fourier transform of $f(x)$ is shown
in Fig. \ref{fig:highdloss}(a), in which we focus on the peak frequencies
 marked by black squares. First, we use the MSE as the training
loss function.

\textbf{Initial stage in Fig. \ref{fig:highdloss} (b).} The ratio
of the change of the loss function, $\abs{\D L_{\eta}^{+}/\D t}/\abs{\D L/\D t}$
in the upper panel, and the ratio of the change of the DNN output,
$\norm{\D \hat{h}/\D t}_{L^{2}(B_{\eta}^c)}/\norm{\D \hat{h}/\D t}_{L^{2}(\R^{d})}$
in the middle panel, both decreases as frequency increases. At such
initial stage, only the relative error of the first peak frequency,
$\abs{\hat{h}-\hat{f}}/\abs{\hat{f}}$, decreases to a small value. 

\textbf{Intermediate stage in Fig. \ref{fig:highdloss} (c).} The
ratio of the change of the loss function in a certain period, $\abs{L_{\eta}^{+}(\theta(T_1))-L_{\eta}^{+}(\theta(T_2))}/\abs{L_{\eta}(\theta(T_1))-L_{\eta}(\theta(T_2))}$,
increases with $\abs{T_2-T_1}$ for a fixed $\eta$. 

\textbf{Final stage in Fig. \ref{fig:highdloss} (d).} There exists
a frequency $\eta_{0}$ --- when $\eta>\eta_{0},$ the ratio of the
change of the loss function, $\abs{\D L_{\eta}^{+}/\D t}/\abs{\D L/\D t}$
in the upper panel, and the ratio of the change of the DNN output,
$\norm{\D \hat{h}/\D t}_{L^2(B_{\eta}^c)}/\norm{\D \hat{h}/\D t}_{L^2(\R^d)}$
in the middle panel, both decreases as frequency increases. At such
final stage, only peak frequencies corresponding to high frequencies have not converged yet.

Secondly, we use the $L^{4}$ training loss $\frac{1}{M}\sum_{i=1}^M(h(x_i,\theta)-y_i)^4$
as shown in Fig. \ref{fig:highdloss-tanh-L4}. We obtain similar results.
\begin{center}
\begin{figure}
\begin{centering}
\subfloat[Target in Fourier domain]{\begin{centering}
\includegraphics[scale=0.45]{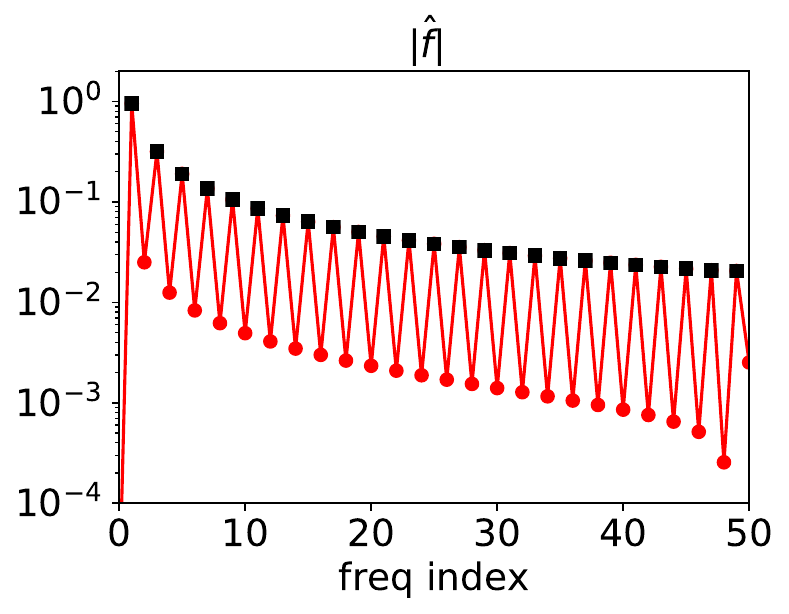}
\par\end{centering}
}\subfloat[Initial stage]{\begin{centering}
\includegraphics[scale=0.45]{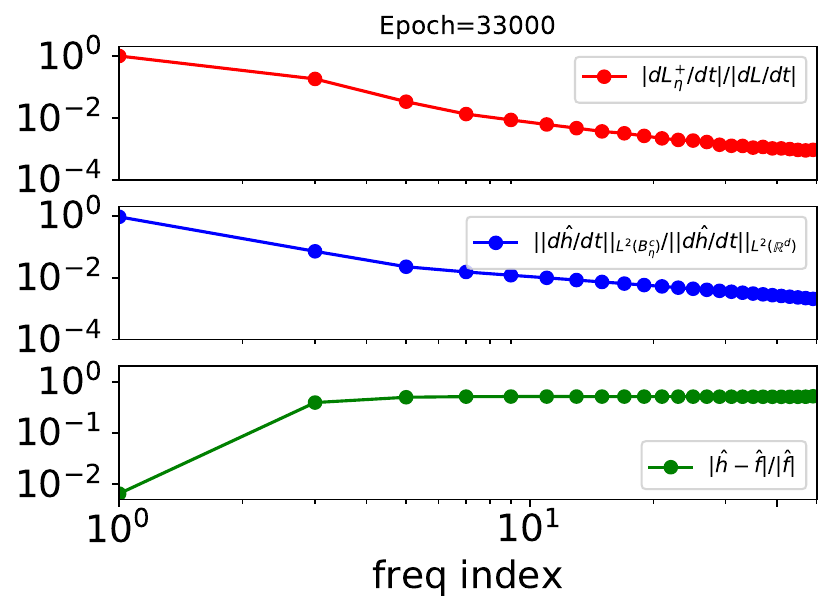}
\par\end{centering}
}
\par\end{centering}
\begin{centering}
\subfloat[Intermediate stage]{\begin{centering}
\includegraphics[scale=0.45]{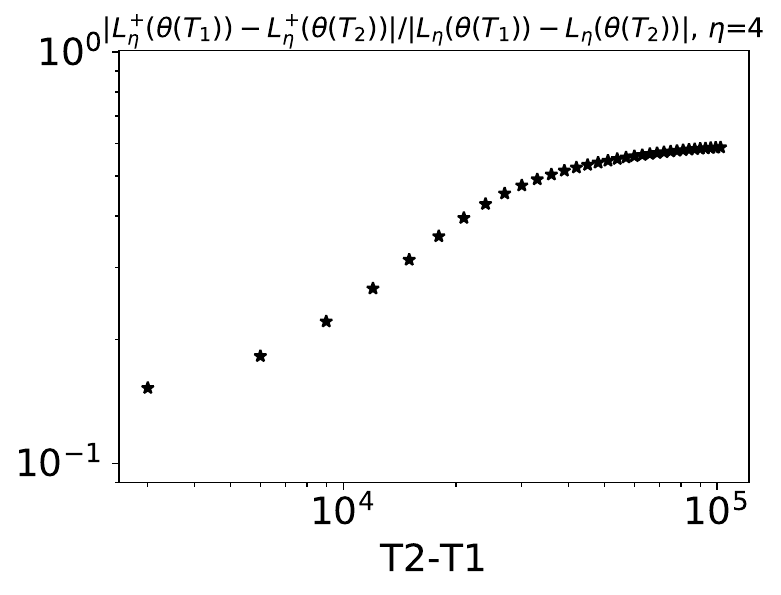}
\par\end{centering}
}\subfloat[Final stage]{\begin{centering}
\includegraphics[scale=0.45]{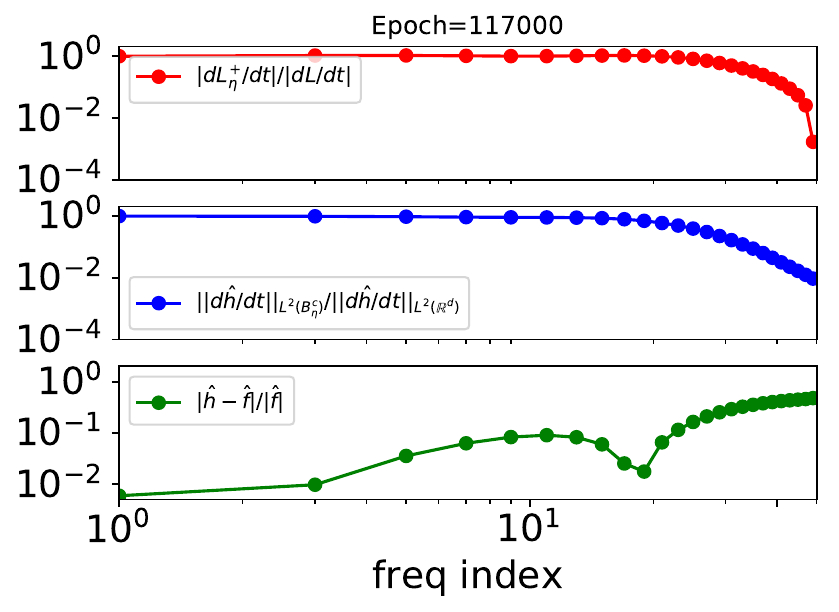}
\par\end{centering}
}
\par\end{centering}
\caption{Numerical understanding of theorems of MSE training loss. (a) Amplitude
of DFT of the training samples against frequency index. Frequencies
marked by black squares are analyzed in the second row. (b, d) upper:
$\abs{\D L_{\protect\geq\eta}/\D t}/\abs{\D L/\D t}$
vs. frequency index. Middle: $\norm{\D \hat{h}/\D t}_{L^{2}(B_{\eta}^{c})}/\norm{\D \hat{h}/\D t}_{L^{2}(\mathbb{R}^{d})}$
vs. frequency index. Lower: Relative error of each selected frequency,
$\abs{\hat{h}-\hat{f}}/\abs{\hat{f}}$ vs. frequency index. Each sub-figure
is plotted at one training epoch. (c) $\abs{L_{\eta}^+(\theta(T_1))-L_{\eta}^+(\theta(T_2))}/\abs{L_{\eta}(\theta(T_1))-L_{\eta}(\theta(T_2))}$
vs. $\abs{T_2-T_1}$ with $\eta$ is selected as the fourth frequency
peak. We use a tanh-DNN with widths 1-200-50-1 with full batch training
by Adam optimizer. The learning rate is $2\times10^{-5}$. \label{fig:highdloss} }
\end{figure}
\par\end{center}

\begin{center}
\begin{figure}
\begin{centering}
\subfloat[Initial stage]{\begin{centering}
\includegraphics[scale=0.3]{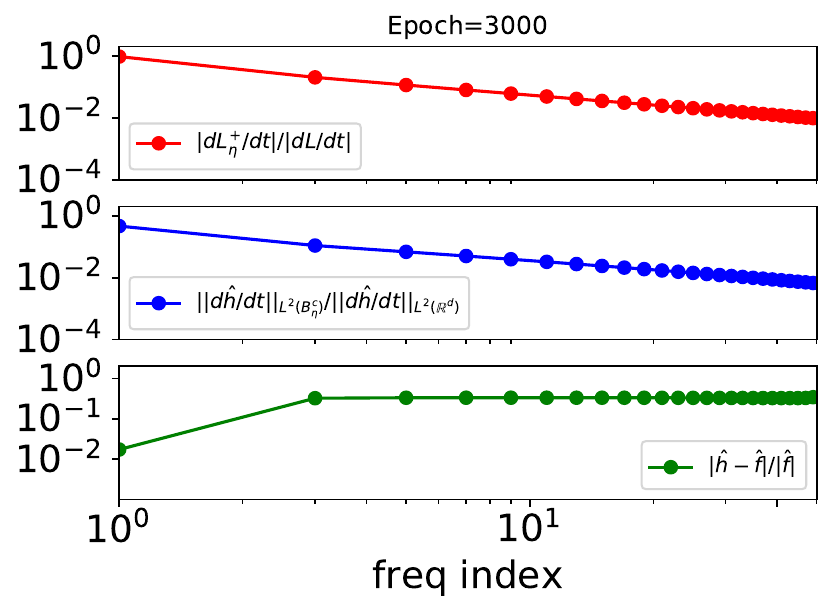}
\par\end{centering}
}\subfloat[Intermediate stage]{\begin{centering}
\includegraphics[scale=0.3]{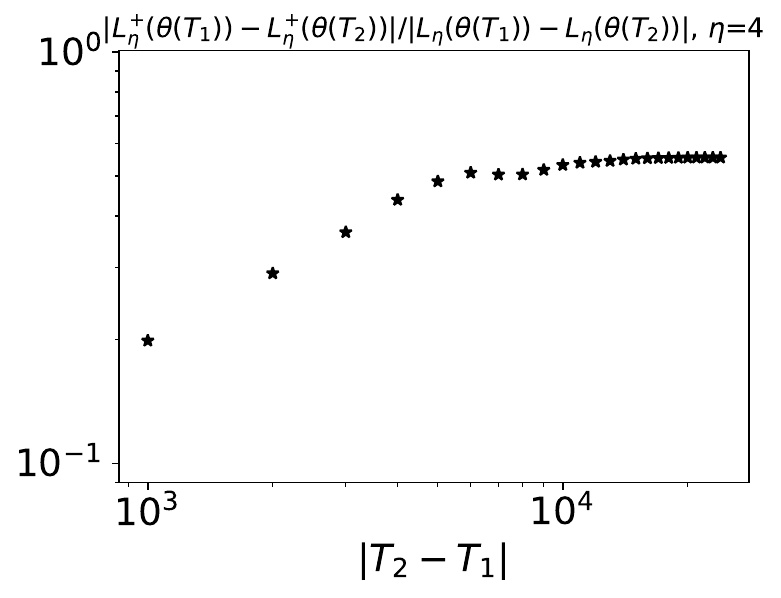}
\par\end{centering}
}\subfloat[Final stage]{\begin{centering}
\includegraphics[scale=0.3]{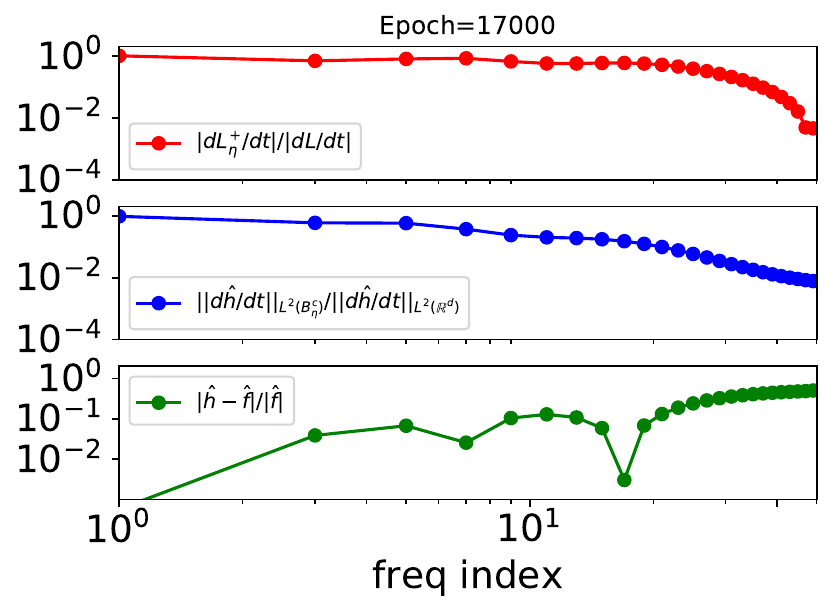}
\par\end{centering}
}
\par\end{centering}
\caption{Numerical understanding of theorems of $L^{4}$ training loss. The
illustrations are same as Fig. \ref{fig:highdloss} (b, c, d), respectively.
We use a tanh-DNN with widths 1-500-500-500-500-1 with full batch
training. \label{fig:highdloss-tanh-L4} }
\end{figure}
\par\end{center}

\section{Proof of Theorems}
\subsection{F-Principle: Initial Stage (Theorem \ref{thm..InitialStage})}

In this section, we focus on the initial stage of the training dynamics. The first result shows that the change of loss function concentrates on low frequencies.

In general, $C$ may depend on $T$. In the next section, we will provide a similar result in some situation where $C$ does not depend on $T$.

\begin{proof}[\textbf{Proof of Theorem \ref{thm..InitialStage} ($L^2$ loss function)}]
    The dynamics for the loss function contributed by high frequency reads as:
    \begin{align} 
        \frac{\D L^+_{\rho,\eta}(\theta)}{\D t}
        &= \left(\int_{B_\eta^c}\nabla_\theta q_\rho(\xi,\theta)\diff{\xi}\right)\cdot\frac{\D \theta}{\D t}\nonumber\\
        &= -\left(\int_{B_\eta^c}\nabla_\theta q_\rho(\xi,\theta)\diff{\xi}\right)\cdot\nabla_\theta L_\rho(\theta).
    \end{align}
    The dynamics for the total loss function is 
    \begin{equation}
        \frac{\D L_\rho(\theta)}{\D t}=-\abs{\nabla_\theta L_\rho(\theta)}^2.
    \end{equation}
    Therefore
    \begin{align}
        \frac{\abs{\D L^+_{\rho,\eta}/\D t}}{\abs{\D L_\rho/\D t}}
        &\leq \frac{\Big(\int_{B_\eta^c}\abs{\nabla_\theta q_\rho(\xi,\theta)}\diff{\xi}\Big)\abs{\nabla_\theta L_\rho(\theta)}}{\abs{\nabla_\theta L_\rho(\theta)}^2}\nonumber\\
        &= \frac{\norm{\nabla_\theta q_\rho(\cdot,\theta)}_{L^1(B_\eta^c)}}{\abs{\nabla_\theta L_\rho(\theta)}}.\label{eq..DecayFiniteTime}
    \end{align}
    Note that $\eta\leq \langle\xi\rangle$ for all $0<\eta\leq \abs{\xi}$. Therefore
    \begin{align}
        \norm{\nabla_\theta q_\rho(\cdot,\theta)}_{L^1(B_\eta^c)}
        &\leq \eta^{-m}\int_{B_\eta^c}\langle\xi\rangle^{m}\abs{\nabla_\theta q_\rho(\xi,\theta)}\diff{\xi}\nonumber\\
        &\leq \eta^{-m}\norm{\langle\cdot\rangle^{m} \abs{\nabla_\theta q_\rho(\cdot,\theta)}}_{L^1(\R^d)}.
    \end{align}
    By Assumption \ref{assump..BoundedTrajectory}, $\sup_{t\geq 0}\abs{\theta(t)}\leq R$ and
    \begin{equation}
      \sup_{t\in(0,T]}\norm{\langle\cdot\rangle^m\nabla_\theta q_\rho(\cdot,\theta)}_{L^1(\R^d)}<+\infty.
    \end{equation}
    For $k=1$, we take the assumption that $\inf_{t\in(0,T]}\abs{\nabla_\theta L_\rho(\theta(t))}>0$.
    For $k\geq 2$, according to Assumption \ref{assump..Differentiability}, we have $\nabla_\theta L_\rho(\cdot)\in W^{1,\infty}_\mathrm{loc}(\R^d;\R^N)$, and hence $\nabla_\theta L_\rho(\theta)$ is locally Lipschitz in $\theta$. This together with Assumption \ref{assump..BoundedTrajectory} implies that $\nabla_\theta L_\rho(\theta(t))$ is continuous on $t\in[0,T]$. If $\inf_{t\in(0,T]}\abs{\nabla_\theta L_\rho(\theta(t))}=0$, then there is a $t_0\in [0,T]$ such that $\abs{\nabla_\theta L_\rho(\theta(t_0))}=0$. By the uniqueness of ordinary differential equation, we have $\abs{\nabla_\theta L_\rho(\theta(T))}=0$ which contridicts with the assumption that $\abs{\nabla_\theta L_\rho(\theta(t))}>0$. Therefore $\inf_{t\in(0,T]}\abs{\nabla_\theta L_\rho(\theta(t))}>0$.
    Thus for $k\geq 1$ the following ratio is bounded from above:
    \begin{equation}
        C:=\sup_{t\in(0,T]} \frac{\norm{\langle\cdot\rangle^{m} \nabla_\theta q_\rho(\cdot,\theta)}_{L^1(\R^d)}}{\abs{\nabla_\theta L_\rho(\theta)}}<+\infty.
    \end{equation}
    Therefore
    \begin{equation}
        \frac{\abs{\D L^+_{\rho,\eta}/\D t}}{\abs{\D L_\rho/\D t}}\leq C\eta^{-m},\quad t\in(0,T].
    \end{equation}
\end{proof}

\begin{cor}[dissipation]
    In the situation of Theorem \ref{thm..InitialStage} for $L^2$ loss function, we have that for sufficiently large $\eta$
    \begin{equation}
        \frac{\D L^-_{\rho,\eta}}{\D t}
        \leq -(1-C\eta^{-m})\abs{\nabla_\theta L_\rho}^2
        \leq 0\label{eq..LowFrequencyLossFunctionDecay}.
    \end{equation}
\end{cor}
\begin{proof}
    For sufficiently large $\eta$, the dynamics of $L^-_{\rho,\eta}$ is dissipative because
    \begin{equation*}
        \frac{\D L^-_{\rho,\eta}(\theta)}{\D t}
        = \frac{\D L_\rho(\theta)}{\D t}-\frac{\D L^+_{\rho,\eta}(\theta)}{\D t}
        \leq -(1-C\eta^{-m})\abs{\nabla_\theta L_\rho(\theta)}^2
        \leq 0.
    \end{equation*}
\end{proof}

Next we prove the case of general loss function.

\begin{proof}[\textbf{Proof of Theorem \ref{thm..InitialStage} (general loss function)}]
    On the one hand, we estimate the numerator by studying the dynamics for $\hat{h}$:
    \begin{equation}
        \frac{\D \hat{h}(\xi,\theta)}{\D t}
        =\nabla_\theta \hat{h}(\xi,\theta)\cdot\frac{\D \theta}{\D t}
        =-\nabla_\theta \hat{h}(\xi,\theta)\cdot\nabla_\theta \tilde{L}_\rho(\theta).\label{eq..dynamicsFtildeh}
    \end{equation}
    Taking square and integrating both sides on $B_\eta^c$ leads to the upper bound on the numerator
    \begin{equation}
        \Norm{\frac{\D \hat{h}(\cdot,\theta)}{\D t}}_{L^2(B_\eta^c)}
        \leq \abs{\nabla_\theta \tilde{L}_\rho(\theta)} \norm{\nabla_\theta \hat{h}(\cdot,\theta)}_{L^2(B_\eta^c)}. \label{eq..hhatGeneralLossNumeratorEstimate}
    \end{equation}
    On the other hand, note the dynamics for the hypothesis function
    \begin{equation}
        \frac{\D h(x,\theta)}{\D t}=\nabla_\theta h(x,\theta)\cdot \frac{\D \theta}{\D t}=-\nabla_\theta \tilde{L}_\rho(\theta)\cdot \nabla_\theta h(x,\theta)
    \end{equation}
    and the dynamics for the total loss function
    \begin{equation}
        \frac{\D \tilde{L}_\rho(\theta)}{\D t}=-\abs{\nabla_\theta\tilde{L}_\rho(\theta)}^2.
    \end{equation}
    Thus we have 
    \begin{align}
        \abs{\nabla_\theta \tilde{L}_\rho(\theta)}^2
        &=\Abs{\frac{\D \tilde{L}_\rho(\theta)}{\D t}}\nonumber\\
        &=\Abs{\frac{\D}{\D t}\int_{\R^d}\ell(h(x,\theta)-f(x))\rho(x)\diff{x}}\nonumber\\
        &=\Abs{\int_{\R^d} \frac{\D h(x,\theta)}{\D t}\ell'(h(x,\theta)-f(x))\rho(x)\diff{x}}\nonumber\\
        &\leq \norm{\sqrt{\rho}}_{L^\infty}\Norm{\frac{\D h(\cdot,\theta)}{\D t}}_{L^2(\R^d)}\norm{\ell'(h(\cdot,\theta)-f(\cdot))\sqrt{\rho(\cdot)}}_{L^2(\R^d)}, \label{eq..hhatGeneralLossDenominatorEstimate}
    \end{align}
    where we used the Cauchy--Schwarz inequality in the last step. Combining Eqs. \eqref{eq..hhatGeneralLossNumeratorEstimate} and \eqref{eq..hhatGeneralLossDenominatorEstimate}, we obtain
    \begin{align}
        \frac{\norm{\frac{\D \hat{h}}{\D t}}_{L^2(B_\eta^c)}}
        {\norm{\frac{\D \hat{h}}{\D t}}_{L^2(\R^d)}}
        &\leq \frac{\norm{\sqrt{\rho}}_{L^\infty}\norm{\ell'(h(\cdot,\theta)-f(\cdot))\sqrt{\rho(\cdot)}}_{L^2(\R^d)}
        \abs{\nabla_\theta \tilde{L}_\rho(\theta)}\norm{\nabla_\theta \hat{h}(\cdot,\theta)}_{L^2(B_\eta^c)}
        }{\abs{\nabla_\theta \tilde{L}_\rho(\theta)}^2}\nonumber\\
        &\leq \norm{\sqrt{\rho}}_{L^\infty}\norm{\nabla_\theta \hat{h}(\cdot,\theta)}_{L^2(B_\eta^c)}\frac{\norm{\ell'(h(\cdot,\theta)-f(\cdot))\sqrt{\rho(\cdot)}}_{L^2(\R^d)}}{\abs{\nabla_\theta \tilde{L}_\rho(\theta)}}.
    \end{align}
        Similar to the case of $L^2$ loss,
    \begin{align}
        \norm{\nabla_\theta \hat{h}(\cdot,\theta)}_{L^2(B_\eta^c)}
        &\leq \eta^{-m} \left(\int_{B_\eta^c}\langle\xi\rangle^{2m}\abs{\nabla_\theta \hat{h}(\xi,\theta)}^2\diff{\xi}\right)^{1/2}\nonumber\\
        &\leq \eta^{-m} \norm{\langle\cdot\rangle^{m} \nabla_\theta \hat{h}(\cdot,\theta)}_{L^2(\R^d)}.
    \end{align}
    Again, by Assumption \ref{assump..BoundedTrajectory}, $\sup_{t\geq 0}\abs{\theta(t)}\leq R$ and 
    \begin{equation}
        \sup_{t\in(0,T]}\norm{\langle\cdot\rangle^{m} \nabla_\theta \hat{h}(\cdot,\theta)}_{L^2(\R^d)}<+\infty.
    \end{equation}
    For $k\geq 2$, the same argument of the $L^2$ loss case leads to $\inf_{t\in(0,T]}\abs{\nabla_\theta \tilde{L}_\rho(\theta(t))}>0$.
    The proof is completed by the following bound
    \begin{align}
        \sup_{t\in(0,T]}\frac{\norm{\ell'(h(\cdot,\theta)-f(\cdot))\sqrt{\rho(\cdot)}}_{L^2(\R^d)}}{\abs{\nabla_\theta \tilde{L}_\rho(\theta)}}
        &\leq \sup_{t\in(0,T]}\frac{\left(C\int_{\R^d}\ell(h(x,\theta)-f(x))\rho(x)\diff{x}\right)^{1/2}}{\abs{\nabla_\theta\tilde{L}_\rho(\theta)}}\nonumber\\
        &< +\infty, \nonumber
    \end{align}
    where we used Assumption \ref{assump..GeneralLoss}.
\end{proof}

\subsection{F-Principle: Intermediate Stage (Theorem \ref{thm..IntermediateStage})} 

In this section, we prove the key theorem for the intermediate stage. This theorem then implies several useful corollaries.

\begin{proof}[\textbf{Proof of Theorem \ref{thm..IntermediateStage}}]
    The numerator can be controlled as follows
    \begin{align}
        &~~\int_{T_1}^{T_2}\Abs{\frac{\D L_{\eta}^+(\theta)}{\D t}}\diff{t}\nonumber\\
        &= \int_{T_1}^{T_2}\Abs{\left(\int_{B_\eta^c}\nabla_\theta q(\xi,\theta(t))\diff{\xi}\right)\cdot\frac{\D \theta(t)}{\D t}}\diff{t}\nonumber\\
        &\leq \int_{T_1}^{T_2}\left( \int_{B_\eta^c}\abs{\nabla_\theta q(\xi,\theta(t))}\diff{\xi}\right) \abs{\nabla_\theta \tilde{L}_\rho(\theta(t))}\diff{t}\nonumber\\
        &=\int_{T_1}^{T_2}\abs{\nabla_\theta \tilde{L}_\rho(\theta(t))}\int_{B_\eta^c}\Abs{\nabla_\theta \hat{h}(\xi,\theta(t))\overline{\hat{h}(\xi,\theta(t))-\hat{f}(\xi)}+\cc}\diff{\xi}\diff{t}\nonumber\\
        &\leq 2\int_{T_1}^{T_2}\abs{\nabla_\theta \tilde{L}_\rho(\theta(t))}\norm{\nabla_\theta\hat{h}(\cdot,\theta(t))}_{L^2(B_\eta^c)}\norm{h(\cdot,\theta(t))-f(\cdot)}_{L^2(\R^d)}\diff{t}\nonumber\\
        &\leq 2\eta^{-m}\left(\sup_{t\in[T_1,T_2]}\norm{\langle\cdot\rangle^{m}\nabla_\theta\hat{h}(\cdot,\theta(t))}_{L^2(\R^d)}\right)\int_{T_1}^{T_2}\abs{\nabla_\theta \tilde{L}_\rho(\theta(t))}L(\theta(t))^{1/2}\diff{t},
    \end{align}
    where in the second-to-last step we used the Cauchy--Schwarz inequality and the Plancherel theorem, and in the last step we used the following
    \begin{equation}
        \norm{\nabla_\theta\hat{h}(\cdot,\theta(t))}_{L^2(B_\eta^c)}\leq \eta^{-m}\norm{\langle\cdot\rangle^{m}\nabla_\theta\hat{h}(\cdot,\theta(t))}_{L^2(\R^d)}.
    \end{equation}
    By Assumption \ref{assump..BoundedTrajectory}, $\sup_{t\geq 0}\abs{\theta(t)}\leq R$ and
    \begin{equation}
        C_1:=\sup_{t\geq 0}\norm{\langle\cdot\rangle^{m}\nabla_\theta\hat{h}(\cdot,\theta(t))}_{L^2(\R^d)}<\infty.
    \end{equation}
    By the assumption that $\frac{1}{2}L(\theta(T_1))\geq L(\theta(T_2))$
    , we have 
    \begin{align}
        \int_{T_1}^{T_2}\Abs{\frac{\D L}{\D t}}\diff{t}
        \geq \abs{L(\theta(T_1))-L(\theta(T_2))}
        \geq \frac{1}{2}L(\theta(T_1)).
    \end{align}
    Therefore,
    \begin{align}
        \frac{\int_{T_1}^{T_2}\Abs{\frac{\D L_{\eta}^+}{\D t}}\diff{t}}{\int_{T_1}^{T_2}\Abs{\frac{\D L}{\D t}}\diff{t}}
        &\leq \frac{2C_1\eta^{-m}\int_{T_1}^{T_2}\abs{\nabla_\theta \tilde{L}_\rho(\theta(t))}L(\theta(t))^{1/2}\diff{t}}{\abs{\frac{1}{2}L(\theta(T_1))}^{1/2}(\int_{T_1}^{T_2}\Abs{\frac{\D L(\theta(t))}{\D t}}\diff{t})^{1/2}}\nonumber\\
        &\leq  
        \frac{2\sqrt{2}C_1\eta^{-m}(\int_{T_1}^{T_2}\abs{\nabla_\theta\tilde{L}_\rho(\theta(t))}^2\diff{t})^{1/2}}{\abs{L(\theta(T_1))}^{1/2}}\times\frac{(\int_{T_1}^{T_2}L(\theta(t))\diff{t})^{1/2}}{(\int_{T_1}^{T_2}\Abs{\frac{\D L(\theta(t))}{\D t}}\diff{t})^{1/2}}.
    \end{align}
    Recall the training dynamics
    \begin{equation}
      \frac{\D \theta}{\D t}
      = -\nabla_\theta \tilde{L}_\rho(\theta),
    \end{equation}
    where for $k\geq 2$, $\nabla_\theta\tilde{L}_\rho(\cdot)\in W^{1,\infty}_\mathrm{loc}(\R^N)\subset C^{0,1}(\R^N)$. Hence $\theta(\cdot)\in C^{0,1}([0,+\infty))$.
    Taking further time derivative of $\theta$, we obtain
    \begin{equation}
      \frac{\D^2 \theta}{\D t^2}
      = -\nabla^2_\theta\tilde{L}_\rho(\theta)\cdot\frac{\D \theta}{\D t}
      =\nabla^2_\theta\tilde{L}_\rho(\theta)\cdot\nabla_\theta\tilde{L}_\rho(\theta).
    \end{equation}
    Since $\nabla_\theta\tilde{L}_\rho(\cdot), \nabla^2_\theta\tilde{L}_\rho(\cdot)\in L^\infty_\mathrm{loc}(\R^N)$ and $\theta(\cdot)$ is continuous, we have $\frac{\D^2 \theta(\cdot)}{\D t^2}\in L^{\infty}_\mathrm{loc}([0,+\infty))$. Taking time derivatives of the $L^2$ loss function $L$, we obtain
    \begin{align}
      \frac{\D L}{\D t}
      &= \nabla_\theta L\cdot\frac{\D \theta}{\D t},\\
      \frac{\D^2 L}{\D t^2}
      &= (\nabla^2_\theta L\cdot\frac{\D \theta}{\D t})\cdot \frac{\D \theta}{\D t}+\nabla_\theta L\cdot\frac{\D^2 \theta}{\D t^2}.
    \end{align}
    The facts that $\nabla^2_\theta L(\cdot), \nabla_\theta L(\cdot)\in L^\infty_\mathrm{loc}(\R^N)$ and that $\theta(\cdot)$ is continuous lead to $\nabla^2_\theta L(\theta(\cdot)), \nabla_\theta L(\theta(\cdot))\in L^\infty_\mathrm{loc}([0,+\infty))$. This with $\frac{\D \theta(\cdot)}{\D t},\frac{\D^2 \theta(\cdot)}{\D t^2}\in L^{\infty}_\mathrm{loc}([0,+\infty))$ implies that $\frac{\D^2 L(\theta(\cdot))}{\D t^2}\in L^{\infty}_\mathrm{loc}([0,+\infty))$. 
    Therefore $\frac{\D L(\theta(\cdot))}{\D t}\in C^{0,1}([0,+\infty))$. 
    Thus $M:=\max_{t\in[T_1,T_2]}L(\theta(t))$ is finite. If $M\leq 2 L(\theta(T_1))$, we have
    \begin{align}
      \frac{\int_{T_1}^{T_2}L(\theta)\diff{t}}{\int_{T_1}^{T_2}\Abs{\frac{\D L(\theta)}{\D t}}\diff{t}}
      &\leq \frac{(T_2-T_1)M}{\abs{L(\theta(T_1))-L(\theta(T_2))}}\nonumber\\
      &\leq \frac{(T_2-T_1)2L(\theta(T_1))}{\frac{1}{2}L(\theta(T_1))}\nonumber\\
      &=4(T_2-T_1).\label{eq..ratioM1}
    \end{align}
    If $M>2L(\theta(T_1))$, then we choose $t_{M}\in[T_1,T_2]$ such that $L(\theta(t_M))=M$. We have
    \begin{align}
      \frac{\int_{T_1}^{T_2}L(\theta)\diff{t}}{\int_{T_1}^{T_2}\Abs{\frac{\D L(\theta)}{\D t}}\diff{t}}
      &\leq \frac{(T_2-T_1)M}{\abs{L(\theta(T_1))-L(
      \theta(t_M))}+\abs{L(\theta(t_M))-L(\theta(T_2))}}\nonumber\\
      &= \frac{(T_2-T_1)M}{M-L(\theta(T_1))+M-L(\theta(T_2))}\nonumber\\
      &\leq \frac{(T_2-T_1)M}{2M-\frac{3}{2}L(\theta(T_1))}\nonumber\\
      &\leq \frac{4}{5}(T_2-T_1).\label{eq..ratioM2}
    \end{align}
    Combining Eqs. \eqref{eq..ratioM1} and \eqref{eq..ratioM2}, we have
    \begin{equation}
      \frac{\int_{T_1}^{T_2}L(\theta)\diff{t}}{\int_{T_1}^{T_2}\Abs{\frac{\D L(\theta)}{\D t}}\diff{t}}
      \leq 4(T_2-T_1).
    \end{equation}
    Therefore
    \begin{align}
        \frac{\int_{T_1}^{T_2}\Abs{\frac{\D L_{\eta}^+}{\D t}}\diff{t}}{\int_{T_1}^{T_2}\Abs{\frac{\D L}{\D t}}\diff{t}}
        &\leq \frac{2\sqrt{8}C_1\eta^{-m}\sqrt{T_2-T_1}\abs{\tilde{L}_\rho(\theta(T_1))-\tilde{L}_\rho(\theta(T_2))}^{1/2}}{\abs{L(\theta(T_1))}^{1/2}}\nonumber\\
        &\leq 4\sqrt{2}C_1\eta^{-m}\sqrt{T_2-T_1}\frac{\abs{\tilde{L}_\rho(\theta(T_1))}^{1/2}}{\abs{L(\theta(T_1))}^{1/2}}\nonumber\\
        &= C\sqrt{T_2-T_1}\eta^{-m},
    \end{align}
    where $C=4\sqrt{2}C_1C_2^{1/2}$ and 
    \begin{equation}
        C_2:=\sup_{t\geq 0}\frac{\abs{\tilde{L}_\rho(\theta(t))}}{\abs{L(\theta(t))}}.
    \end{equation}
    Now it is sufficient to show that $C_2<+\infty$. In fact, there is a constant $C_3$ such that $\sup_{\abs{\theta}\leq R}\sup_{x\in\R^d}\abs{h(x,\theta)-f(x)}\leq C_3$. This with Assumption \ref{assump..GeneralLoss} implies that $\ell(z)\leq C_4\abs{z}^2$ for $\abs{z}\leq C_3$. Therefore
    \begin{align}
        C_2
        &\leq \sup_{t\geq 0}\frac{\abs{\int_{\R^d}\ell(h(x,\theta(t))-h(x,\theta^*))\rho(x)\diff{x}}}{\abs{\int_{\R^d}(h(x,\theta(t))-h(x,\theta^*))^2\diff{x}}}\nonumber\\
        &\leq \norm{\rho}_{L^\infty}
        \sup_{t\geq 0}\frac{\abs{\int_{\R^d}C_4(h(x,\theta(t))-h(x,\theta^*))^2\diff{x}}}{\abs{\int_{\R^d}(h(x,\theta(t))-h(x,\theta^*))^2\diff{x}}}<+\infty.
    \end{align}
\end{proof}
\begin{rmk}
  If the condition $\frac{1}{2}L(\theta(T_1))\geq L(\theta(T_2))$ is replaced by $\delta L(\theta(T_1))\geq L(\theta(T_2))$ for any $\delta\in(0,1)$, the estimates in Theorem \ref{thm..IntermediateStage} and the following corollaries still hold. 
\end{rmk}
\begin{cor}\label{cor..DifferenceRatioGeneralLoss}
    Under the same assumptions in Theorem \ref{thm..IntermediateStage}, for any $1\leq m\leq k-1$, 
    there is a constant $C>0$ such that for any $0<T_1<T_2$ satisfying $\frac{1}{2}L(\theta(T_1))\geq L(\theta(T_2))$ and $L(\theta(T_1))\geq L(\theta(t))$ for all $t\in[T_1,T_2]$, we have
    \begin{equation}
      \frac{\abs{L_{\eta}^+(\theta(T_1))-L_{\eta}^+(\theta(T_2))}}{\abs{L(\theta(T_1))-L(\theta(T_2))}}\leq C\sqrt{T_2-T_1}\eta^{-m}.\label{eq..SqrtTGeneralLossDifferenceRatio}
    \end{equation}
\end{cor}
\begin{proof}
  Similar to the proof of Theorem \ref{thm..IntermediateStage}, we have the upper bound for the numerator
  \begin{equation}
    \abs{L_{\eta}^+(\theta(T_1))-L_{\eta}^+(\theta(T_2))}\leq 2\eta^{-m} C_1\int_{T_1}^{T_2}\abs{\nabla_\theta \tilde{L}_\rho(\theta(t))}L(\theta(t))^{1/2}\diff{t}
  \end{equation}
  and lower bound for the denominator
  \begin{equation}
    \abs{L(\theta(T_1))-L(\theta(T_2))}\geq \frac{1}{2}L(\theta(T_1))
  \end{equation}
  with $C_1:=\sup_{t\in[T_1,T_2]}\norm{\langle\cdot\rangle^{m}\nabla_\theta\hat{h}(\cdot,\theta(t))}_{L^2(\R^d)}<+\infty$.
  Therefore, these bounds with the assumption that $L(\theta(T_1))\geq L(\theta(t))$ for all $t\in[T_1, T_2]$ leads to
  \begin{align}
      \frac{\abs{L_{\eta}^+(\theta(T_1))-L_{\eta}^+(\theta(T_2))}}{\abs{L(\theta(T_1))-L(\theta(T_2))}}
      &\leq \frac{2C_1\eta^{-m}\int_{T_1}^{T_2}\abs{\nabla_\theta \tilde{L}_\rho(\theta(t))}L(\theta(t))^{1/2}\diff{t}}{\frac{1}{2}L(\theta(T_1))}\nonumber\\
      &\leq  
      \frac{4C_1\eta^{-m}\int_{T_1}^{T_2}\abs{\nabla_\theta\tilde{L}_\rho(\theta(t))}\diff{t}}{\abs{L(\theta(T_1))}^{1/2}}\nonumber\\
      &\leq C\sqrt{T_2-T_1}\eta^{-m},
  \end{align}
  where the last inequality is due to the same reason as Theorem \ref{thm..IntermediateStage}.
\end{proof}
\begin{cor}
    Under the same assumptions in Theorem \ref{thm..IntermediateStage}, if the solution $\theta$ converges to a non-degenerate global minimizer $\theta^*$, then for any $1\leq m\leq k-1$, the above upper bound can be improved to the following: there is a constant $C>0$ such that for any $T>0$, we have
    \begin{equation}
        \frac{\int_{0}^{T}\Abs{\frac{\D L_{\eta}^+}{\D t}}\diff{t}}{\int_{0}^{T}\Abs{\frac{\D L}{\D t}}\diff{t}}\leq C\eta^{-m}
    \end{equation}
    and
    \begin{equation}
        \frac{\abs{L_{\eta}^+(\theta(0))-L_{\eta}^+(\theta(T))}}{\abs{L(\theta(0))-L(\theta(T))}}
        \leq C\eta^{-m}.
    \end{equation}
\end{cor}
We skip the proof since this corollary can be obtained directly from Theorem \ref{thm..FinalStage}.

\subsection{F-Principle: Final Stage (Theorem \ref{thm..FinalStage})}
In this section, we prove the F-Principle in final stage of the training dynamics. 

\begin{proof}[\textbf{Proof of Theorem \ref{thm..FinalStage} ($L^2$ loss function)}]
    Following \eqref{eq..DecayFiniteTime}, we have
  \begin{align}
    \frac{\abs{\D L^+_{\rho,\eta}/\D t}}{\abs{\D L_\rho/\D t}}
    &\leq \frac{\int_{B_\eta^c}\Abs{\nabla_\theta \hat{h}_\rho(\xi,\theta)\overline{\hat{h}_\rho(\xi,\theta)-\hat{f}_\rho(\xi)}+\cc}\diff{\xi}}{\abs{\nabla_\theta L_\rho(\theta)}}\nonumber\\
    &\leq \frac{2\norm{\nabla_\theta \hat{h}_\rho(\cdot,\theta)}_{L^2(B_\eta^c)}\norm{\hat{h}_\rho(\cdot,\theta)-\hat{f}_\rho(\cdot)}_{L^2(\R^d)}}{\abs{\nabla_\theta L_\rho(\theta)}},\nonumber\\
    &= 2\norm{\nabla_\theta \hat{h}_\rho(\cdot,\theta)}_{L^2(B_\eta^c)}\frac{\abs{L_\rho(\theta)}^{1/2}}{\abs{\nabla_\theta L_\rho(\theta)}},
  \end{align}
  where we used $\norm{\hat{h}_\rho(\cdot,\theta)-\hat{f}_\rho(\cdot)}_{L^2(\R^d)}^2= L_\rho(\theta)$ in the last inequality. Similar to the local-in-time situation, 
  \begin{align}
    \norm{\nabla_\theta \hat{h}_\rho(\cdot,\theta)}_{L^2(B_\eta^c)}
    &\leq \eta^{-m} \left(\int_{B_\eta^c}\langle\xi\rangle^{2m}\abs{\nabla_\theta \hat{h}_\rho(\xi,\theta)}^2\diff{\xi}\right)^{1/2}\nonumber\\
    &\leq \eta^{-m} \norm{\langle\cdot\rangle^{m} \nabla_\theta\hat{h}_\rho(\cdot,\theta)}_{L^2(\R^d)}.
  \end{align}
  By Assumption \ref{assump..BoundedTrajectory}, $\sup_{t\geq 0}\abs{\theta(t)}\leq R$ and
  \begin{equation}
    \sup_{t\in(0,+\infty)}\norm{\langle\cdot\rangle^{m} \nabla_\theta\hat{h}_\rho(\cdot,\theta(t))}_{L^2(\R^d)}<+\infty.
  \end{equation}
  Now it is sufficient to prove that 
  \begin{equation}
    C:=\lim_{t\to+\infty}\frac{\abs{L_\rho(\theta)}^{1/2}}{\abs{\nabla_\theta L_\rho(\theta)}}<+\infty.
  \end{equation}
  This is true because 
  \begin{align}
    C
    &= \lim_{\theta\to\theta^*}\frac{\abs{L_\rho(\theta)}^{1/2}}{\abs{\nabla_\theta L_\rho(\theta)}}\nonumber\\
    &= \lim_{\theta\to\theta^*}\frac{\abs{(\theta-\theta^*)^T\Lambda(\theta-\theta^*)+o(\abs{\theta-\theta^*}^2)}^{1/2}}{\abs{2\Lambda(\theta-\theta^*)}+o(\abs{\theta-\theta^*})}\nonumber\\
    &<+\infty,
  \end{align}
  where we used the assumption that the minimizer is non-degenerate with the Hessian $\Lambda=\nabla_\theta^2 L_\rho(\theta^*)$.
\end{proof}

Now we finish the proof for general loss function.

\begin{proof}[\textbf{Proof of Theorem \ref{thm..FinalStage} (general loss function)}]
    By the proof of Theorem \ref{thm..InitialStage}, we have
    \begin{align}
        \frac{\norm{\frac{\D \hat{h}}{\D t}}_{L^2(B_\eta^c)}}
        {\norm{\frac{\D \hat{h}}{\D t}}_{L^2(\R^d)}}
        &\leq 2\norm{\sqrt{\rho}}_{L^\infty}\norm{\nabla_\theta \hat{h}(\cdot,\theta)}_{L^2(B_\eta^c)}\frac{\norm{\ell'(h(\cdot,\theta)-f(\cdot))\sqrt{\rho(\cdot)}}_{L^2(\R^d)}}{\abs{\nabla_\theta \tilde{L}_\rho(\theta)}}
    \end{align}
    and
    \begin{align}
        \norm{\nabla_\theta \hat{h}(\cdot,\theta)}_{L^2(B_\eta^c)}
        &\leq \eta^{-m} \norm{\langle\cdot\rangle^{m} \nabla_\theta \hat{h}(\cdot,\theta)}_{L^2(\R^d)}.
    \end{align}
    Since $\lim_{t\to+\infty}\theta(t)=\theta^*$, we have
    \begin{equation}
        \sup_{t\in(0,+\infty)}\norm{\langle\cdot\rangle^{m} \nabla_\theta \hat{h}(\cdot,\theta)}_{L^2(\R^d)}<+\infty.
    \end{equation}
    Now it is sufficient to prove that 
    \begin{equation}
        \sup_{t\in(0,\infty]}\frac{\norm{\ell'(h(\cdot,\theta)-f(\cdot))\sqrt{\rho(\cdot)}}_{L^2(\R^d)}}{\abs{\nabla_\theta \tilde{L}_\rho(\theta)}}<+\infty.
    \end{equation}
    This is true because 
    \begin{align}
        \lim_{t\to+\infty}\frac{\norm{\ell'(h(\cdot,\theta)-f(\cdot))\sqrt{\rho(\cdot)}}_{L^2(\R^d)}}{\abs{\nabla_\theta \tilde{L}_\rho(\theta)}}
        &= \lim_{t\to+\infty}\frac{\left(\int_{\R^d}[\ell'(h(x,\theta)-f(x))]^2\rho(x)\diff{x}\right)^{1/2}}{\abs{\nabla_\theta\tilde{L}_\rho(\theta)}}\nonumber\\
        &\leq \lim_{t\to+\infty}\frac{\left(C\int_{\R^d}\ell(h(x,\theta)-f(x))\rho(x)\diff{x}\right)^{1/2}}{\abs{\nabla_\theta\tilde{L}_\rho(\theta)}}\nonumber\\
        &= \lim_{t\to+\infty}\frac{C^{1/2} \abs{\tilde{L}_\rho(\theta)}^{1/2}}{\abs{\nabla_\theta\tilde{L}_\rho(\theta)}}\nonumber\\
        &= \lim_{\theta\to\theta^*}\frac{\abs{(\theta-\theta^*)^T\tilde{\Lambda}(\theta-\theta^*)+o(\abs{\theta-\theta^*}^2)}^{1/2}}{\abs{2\tilde{\Lambda}(\theta-\theta^*)}+o(\abs{\theta-\theta^*})}\nonumber\\
        &<+\infty,
    \end{align}
    where we used Assumption \ref{assump..GeneralLoss} and the assumption that the minimizer is non-degenerate with the Hessian $\tilde{\Lambda}=\nabla_\theta^2 \tilde{L}_\rho(\theta^*)$.
\end{proof}

\bibliographystyle{icml2019}
\bibliography{DLRef}
\end{document}